\documentclass{article}       

\usepackage{times}
\usepackage{amssymb}
\usepackage{amsmath}
\usepackage{amsfonts}
\usepackage{graphicx}
\usepackage{caption}
\usepackage{color}
\usepackage{url}
\usepackage{bm}
\usepackage{fixmath}
\usepackage{subfigure}
\usepackage{multirow}
\usepackage{amsthm}

\newtheorem{lemma}{Lemma}
\newtheorem{theorem}{Theorem}
\newtheorem{definition}{Definition}

\newtheorem{corollary}{Corollary}
\newtheorem{remark}{Remark}

\begin{document}

\title{On the Theory of Dynamic Graph Regression Problem}


\author{Mostafa Haghir Chehreghani \\
       mostafa.chehreghani@gmail.com \\
       Department of Computer Engineering and \\
       Information Technology, \\
       Amirkabir University of Technology, \\
       Tehran, Iran}

\date{}

\maketitle

\begin{abstract}
Most of real-world graphs are {\em dynamic}, i.e., they change over time by a sequence of update operations.
While the regression problem
has been studied for {\em static} graphs and {\em temporal} graphs,
it is not investigated for general {\em dynamic} graphs.
In this paper,
we study regression over
dynamic graphs.
First, we present the notion of {\em update-efficient matrix embedding},
that defines conditions sufficient for a matrix embedding
to be effectively used for dynamic graph regression
(under $l_2$ norm).
Then, we show that given a $n \times m$
update-efficient matrix embedding (e.g., the adjacency matrix) and after an update operation in the graph,
the exact optimal solution of linear regression can be updated in $O(nm)$ time for the revised graph.
Moreover, we show that this also holds when the matrix embedding is the Laplacian matrix and the update operations are restricted to edge insertion/deletion.
In the end, by conducting experiments over synthetic and real-world graphs,
we show the high efficiency of
updating the solution of graph regression.
\end{abstract}


\section{Introduction}
\label{sec:introduction}

Graphs are an important tool
for modeling many complex systems
such as the world wide web, social networks, road networks
and citation networks.
Most of real-world graphs are dynamic, i.e., they change over time by a sequence of {\em graph update operation}s.
A {\em graph update operation} can be either a node insertion, or a node deletion, or
an edge insertion, or an edge deletion or an edge weight change.

A fundamental task in machine learning and data analysis
is {\em linear regression}.
In this problem, we receive $n$ data,
where for each $i \in [1,n]$,
the data consists of a row in a $n \times m$ matrix $\mathbold{A}$ and a single element in a $n \times 1$ vector $\mathbold{b}$.
Then, the goal is to find a vector $\mathbold{x}$ such that
$\mathbold{A} \cdot \mathbold{x}$ is the closest point to $\mathbold{b}$ in the column span of $\mathbold{A}$,
under a proper distance measure, such as
the $l_2$ norm (Euclidean distance or the least squares distance)
or the $l_1$ norm (the least absolute deviation).
More formally, we want to solve the following problem
($p \in [1, \infty)$):
\begin{equation}
\label{eq:regression}
argmin_{\mathbold{x}} ||\mathbold{A} \cdot \mathbold{x} - \mathbold{b}||_p.
\end{equation}

While this problem has been studied for (static) high dimensional data, static graphs \cite{ecf352e8da75424b8b40be6f22ade69d}
and temporal graphs \cite{DBLP:conf/sdm/HanCSGO19},
to the best of our knowledge
it is not investigated for {\em general dynamic graphs}.
We refer to the regression problem over dynamic graphs as
{\em dynamic graph regression}.
In the current paper,
we study the theory of dynamic graph regression, over general graphs.
The importance of this problem
stems from a wide range of applications that generate dynamic graphs
(hence, many data analysis algorithms have been
developed for them
\cite{DBLP:conf/icdm/BorgwardtKW06,DBLP:journals/pvldb/HayashiAY15,DBLP:conf/hipc/MakkarBG17}).
Examples of these graphs include the world wide web, social networks and collaboration networks.
We focus on the most common form of Equation~\ref{eq:regression}
wherein $p$ is $2$, and is
called {\em least squares error} (or {\em $l_2$ norm}) linear regression.

The challenges of regression  over dynamic graphs
are twofold.
The minor challenge is a proper adaptation of
the standard setting of the regression problem to graphs.
As we will discuss later, this can be done by using
a {\em matrix embedding} for the graph.
The major challenge arises due to the dynamic nature of the graph.
This means when the structure of the graph changes
by an {\em update operation},
the already found solution must be updated,
to become valid for the revised graph.
This is not trivial, as updating the solution must be done in a time
considerably less than the time of computing it from scratch,
for the revised graph.

In the current paper,
we aim to address these challenges.
First, we present the notion of {\em update-efficient matrix embedding}
that defines conditions sufficient for a matrix embedding
to be used for the dynamic graph regression problem
(under $l_2$ norm).
We show that some of the standard matrix embeddings,
e.g., the (weighted) adjacency matrix, satisfy these conditions.
Then, we show that given a $n \times m$
update-efficient matrix embedding,
after an update operation in the graph,
the exact optimal solution of the graph regression problem for the revised graph can be computed in $O(nm)$ time.
In particular, using (weighted) adjacency matrix
as the matrix embedding of the graph,
it takes $O(n^2)$ time to update the optimal solution,
where $n$ is the number of nodes of the revised graph.
Note that in this situation,
computing the optimal solution
for the revised graph from scratch will take $O(n^3)$ time.
We also show that similar results hold
for the Laplacian matrix,
if the update operations are limited to
edge insertion and edge deletion.
In the end, by performing experiments over synthetic and real-world graphs,
we show the high empirical efficiency of the update method.


The rest of this paper is organized as follows.
In Section~\ref{sec:relatedwork},
we provide an overview on related work.
In Section~\ref{sec:preliminaries},
we introduce preliminaries and definitions used in the paper.
In Section~\ref{sec:update-efficient},
we introduce {\em update-efficient} matrix embeddings.
In Section~\ref{sec:graph_regression},
we study dynamic graph regression
under least squares error.
In Section~\ref{sec:experiments},
we present our experimental results.
Finally, the paper is concluded in Section~\ref{sec:conclusion}.

\section{Related work}
\label{sec:relatedwork}

In recent years, several algorithms have been proposed for different
learning problems over nodes of a static graph.
Kleinberg and Tardos~\cite{DBLP:journals/jacm/KleinbergT02} studied
the theory of the node classification problem and
showed its connection to Markov random fields.
With $k$ the number of node labels,
they presented an $O(\log k \log \log k)$-approximation algorithm
which has a polynomial time complexity.
Kovac and Smith
\cite{ecf352e8da75424b8b40be6f22ade69d} developed
a non-parametric regression model for nodes of a graph, where the
distance between estimated values (vector $\mathbf A \mathbf x$) and measured values (vector $\mathbf b$) is computed using $l_2$ norm.
Han et al. \cite{DBLP:conf/sdm/HanZGVO16}
proposed a representation
learning method optimized for
regression over nodes of a graph.
Unlike these algorithms that deal with static graphs, in this paper we concentrate
on efficiently {\em updating} the exact solution of least squares linear regression,
over {\em dynamic} graphs.

There are a number of relevant problems in the literature.
One of them is
regression over {\em temporal graphs}, wherein
node features (node contents) and measured values
change over time, but the structure of the graph remains unchanged.
This is in contrast with our studied problem which deals with changes in the structure of the graph.
For example, while in {\em dynamic graph regression} we want to efficiently update the regression solution after adding a new edge to the graph,
in {\em temporal graph regression} it is desired to update the solution when the values of the \textsf{income} feature (which is one of node features) of some nodes change.
A common technique used to solve {\em temporal graph regression} is
to utilize structural dependencies among measured values~\cite{DBLP:conf/nips/QinLZWL08,DBLP:journals/jmlr/SohnK12,DBLP:conf/icml/WytockK13}.
Han et al.~\cite{DBLP:conf/sdm/HanCSGO19}
presented a different approach
which is based on
jointly learning embeddings for the measured values and
node features.

The other relevant problem is  {\em online learning}
over nodes of a graph~\cite{DBLP:conf/nips/HerbsterP06,DBLP:journals/jmlr/HerbsterPP15}.
{\em Online learning} is used where it is computationally infeasible to solve the learning task over the entire dataset.
However in our studied dynamic graph regression problem, the task
can be solved over the entire graph and the challenge
is to effectively update the solution by structural changes in the graph.
In one of known online graph learning algorithms,
Herbster and Pontil~\cite{DBLP:conf/nips/HerbsterP06} used a perceptron for online label prediction of nodes of a graph.


Another problem which has some connection to our
studied problem is
learning embeddings (representations) for nodes or subgraphs of a graph.
In this problem, each node or subgraph is mapped to a vector in a low-dimensional vector space.
In recent years, several algorithms
have been proposed for it
 \cite{DBLP:conf/kdd/GroverL16,DBLP:conf/icml/YangCS16,DBLP:conf/icml/NiepertAK16,nmi},
even though this task dates back
to several decades ago.
For example,
Parsons and Pisanski
\cite{PARSONS1989143} presented vector embeddings for nodes of a graph such that
the inner product of the vector embeddings of any two nodes $i$ and $j$ is negative iff
$i$ and $j$ are connected by an edge;
and it is $0$ otherwise.
For an overview on learning embeddings for dynamic graphs,
interested readers
may refer to
\cite{DBLP:conf/icdm/LiHJL16,DBLP:journals/corr/abs-1805-11273,DBLP:journals/kbs/GoyalCC20}.

While the above mentioned methods learn with
nodes or edges of a single graph,
there also exist several algorithms in the literature
that learn with a population of graphs \cite{DBLP:journals/ida/ChehreghaniRLC07,DBLP:journals/ml/SaigoNKKT09,DBLP:journals/ida/ChehreghaniCLRG09,DBLP:conf/kdd/LeeRK18,10.1093/imaiai/iaz026,CALISSANO2022104950}.
For example, Calissano et al.~\cite{CALISSANO2022104950} studied building a regression model between a set of real values and a set of graphs.

Our studied setting is usually called {\em dynamical setting} ~\cite{10.1002/widm.1393},
wherein the structure of a single graph changes by means of a {\em graph update operation}, and
the goal is to effectively update the solution of a problem for the new graph.
The list of {\em graph update operations} that change the structure of the graph,
is presented in Section~\ref{sec:preliminaries}.
For a survey on different machine learning and data mining
algorithms in the {\em dynamical setting}, interested readers may refer to~\cite{10.1002/widm.1393}.

\section{Preliminaries}
\label{sec:preliminaries}

We use lowercase letters for scalars,
uppercase letters for constants and graphs,
bold lowercase letters for vectors
and bold uppercase letters for matrices.
We denote by $n$ the number of nodes of graph $G$.
A {\em dynamic graph} is a graph that
changes over time by
a sequence of {\em graph update operations}~\cite{10.1002/widm.1393}.
A {\em graph update operation} is
an operation that inserts an edge or a node into the graph;
or deletes an edge or a node and its incident edges from the graph;
or changes the weight of an edge.
We assume that when a new node in inserted,
some edges are also added between the new node and the existing nodes of the graph.
We also assume that the new node obtains the largest node id of the graph.

The {\em weighted adjacency matrix} of $G$, denoted with $\mathbold{W}$,
is a
$n \times n$ matrix
where $\mathbold{W}_{ij}$ contains the weight
of the edge from node $i$ to node $j$, if $i$ has an edge to $j$
(otherwise, it is $0$).
Given an undirected weighted graph $G$,
its {\em weighted Laplacian matrix} $\mathbold{L}$
is a square matrix of size $n \times n$ defined as
$\mathbold{L} = \mathbold{D} - \mathbold{W}$,
where $\mathbold{D}$ is a $n \times n$ diagonal matrix
with $\mathbold{D}_{ii}=\sum_{j=1}^n \mathbold{W}_{ij}$
and is called the (weighted) degree matrix.
The Euclidean norm or $l_2$ norm of a vector $\mathbold{v}$ of size $n$,
denoted with $||\mathbold{v}||_2$, is defined as
$\sqrt{\mathbold{v}_1^2+\mathbold{v}_2^2+\cdots+\mathbold{v}_n^2}.$
Let $\mathbold{A} \in \mathbb{R}^{n \times m}$.
The rank of $\mathbold{A}$ is the dimension of its column space (or its row space).
By $\mathbold{A}^*$ we denote the transpose of
$\mathbold{A}$ defined as
an operator that switches the row and column indices of $\mathbold{A}$.
The Singular Value Decomposition (SVD) of the $n \times m$ matrix $\mathbold{A}$
is defined as $\mathbold{U} \cdot \mathbold{\Sigma} \cdot \mathbold{V}^*$,
where $\mathbold{U}$ is a $n \times m$ matrix with orthonormal columns,
$\mathbold{\Sigma}$ is a $m \times m$ diagonal matrix with non-negative non-increasing entries on the diagonal,
and $\mathbold{V}^*$ is a $m \times m$ matrix with orthonormal rows.
The Moore-Penrose pseudoinverse matrix of $\mathbold{A}$, denoted by $\mathbold{A}^\dag$,
is the $m \times n$ matrix $\mathbold{V} \cdot \mathbold{\Sigma}^{\dag} \cdot \mathbold{U}^*$,
where $\mathbold{\Sigma}^{\dag}$ is a $m \times m$ diagonal matrix
defined as follows:
$\mathbold{\Sigma}^{\dag}_{i,i} = 1/\mathbold{\Sigma}_{i,i} $, if $\mathbold{\Sigma}_{i,i}>0$,
and $0$ otherwise.
It is well-known that the solution
\begin{equation}
\label{eq:closedform}
\mathbold{x} = \mathbold{A}^{\dag} \cdot \mathbold{b}
\end{equation}
is an optimal solution (the closed form solution) for the linear regression problem under the $l_2$ norm,
i.e., for the problem defined by Equation~\ref{eq:regression}, with $p=2$~\cite{general_inverse_book}.
Time complexity of computing this solution is cubic in terms of $n$ and $m$~\cite{general_inverse_book},
which makes it slow and time consuming to compute it from scratch over large graphs.

\section{Update-efficient matrix embeddings}
\label{sec:update-efficient}

By assuming that data are given in the form of
a graph,
we can extend the linear regression problem to the {\em linear graph regression} problem.
In {\em linear graph regression},
we are given a graph $G$, with $n$ nodes,
and a $n \times 1$ vector $\mathbold{b}$.
Then, we want to solve the following optimization problem:
$$argmin_\mathbold{x} ||{G} \cdot \mathbold{x} - \mathbold{b}||_2^2,$$
where $G \cdot \mathbold{x}$ must satisfy the following conditions:
i) $G \cdot \mathbold{x}$ is well-defined, and
ii) $G \cdot \mathbold{x}$ must produce a $n \times 1$ column vector.
A straightforward and common way to define $G \cdot \mathbold{x}$ is to replace $G$ by vector embeddings of nodes of $G$
(a matrix embedding\footnote{{We note that
this notion of {\em embedding} is different from the notion of {\em embedding} used in {\em graph pattern mining} \cite{DBLP:journals/fgcs/ChehreghaniABB20,DBLP:journals/datamine/ChehreghaniB16,DBLP:journals/tsmc/ChehreghaniCLR11}.}} of $G$),
denoted with $\mathbold{M}$.
As a result, for each node in the graph, we define a $1 \times m$ row vector,
and $\mathbold{x}$ is defined as a $m \times 1$ column vector.
Hence, the linear graph regression problem is converted into finding the
closest point to $\mathbold{b}$, in the column span of the matrix generated
by the vector embeddings of nodes of $G$.
In other words, we want to solve the following optimization problem
\begin{equation}
\label{eq:graphregressiob}
argmin_\mathbold{x}|| \mathbold{M} \cdot \mathbold{x} - \mathbold{b}||_2^2.
\end{equation}

As an example motivating linear graph regression problem,
assume that we are given the graph of a social network,
wherein each vertex is a person and the links represent
the friendship relations.
Moreover, a score is assigned to each vertex
which determines e.g., its reputation.
Now we want to find
a function (with a least squares error) for the scores of the vertices,
which is linear in terms of their structural properties.
Therefore, we need to solve the
linear graph regression problem
for the given network.
More precisely, first we need to compute
a matrix embedding $\mathbold M$ of the nodes of the network, wherein each row $i$ is a representation
of the structural properties of node $i$.
Then we need to find a function for the scores
which is linear in terms of rows of $\mathbold M$.

A property seen in real-world graphs is that they are usually {\em dynamic}.
This means they frequently obtain/loose nodes or edges.
As a result, after an {\em update operation}
the solution found for the linear graph regression problem must be updated.
This should be done in a time much less than
the time of computing it from scratch,
for the updated graph.
In this section, we discuss properties that matrix embeddings
suitable for updating the solution of graph regression should satisfy.
Before that, using an example we discuss what happens to a
matrix embedding, when an update operation occurs in the graph.


\begin{figure*}
\centering
\subfigure[Node deletion.]
{
\includegraphics[scale=0.5]{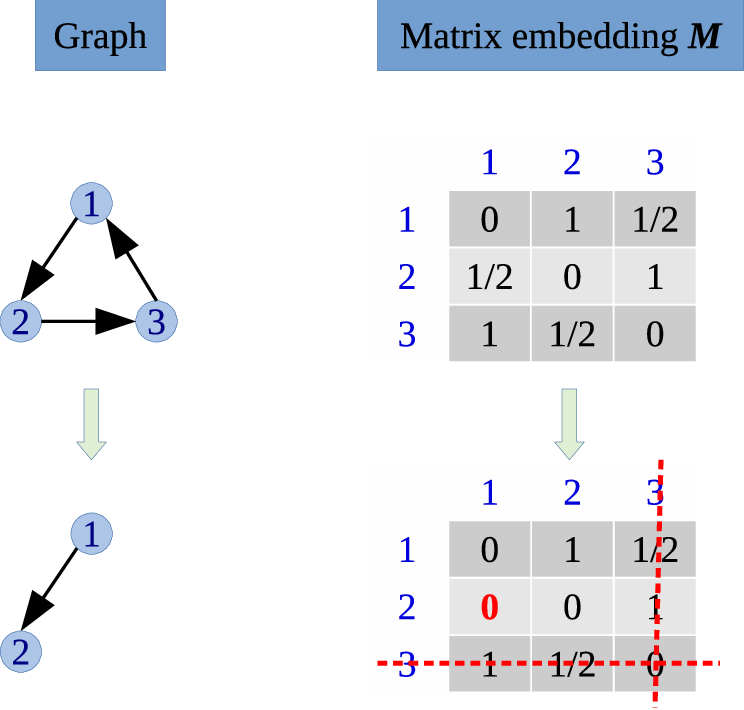}
\label{fig:11}
}
\quad
\subfigure[Node insertion.]
{
\includegraphics[scale=0.5]{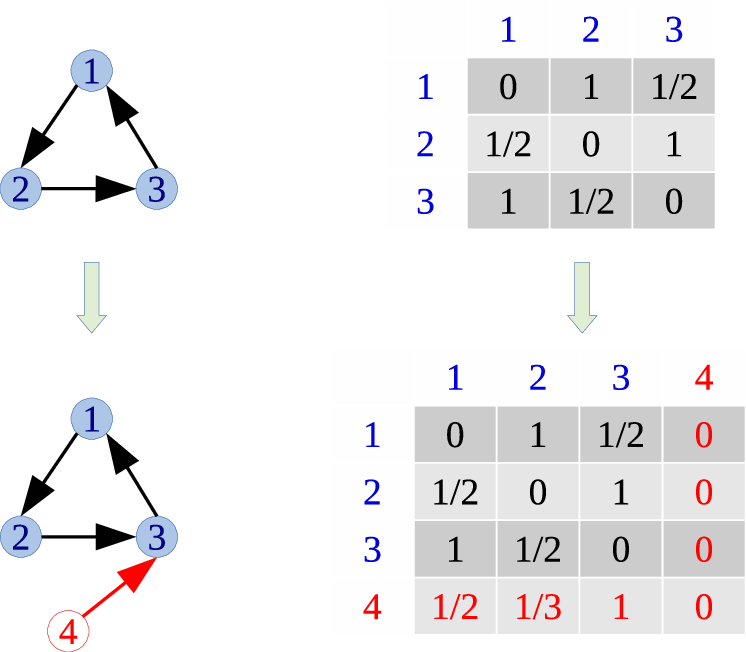}
\label{fig:12}
}
\caption
{
\label{fig:1}
Examples of updating the matrix embedding of a directed graph $G$,
after a node deletion and a node insertion.
}
\end{figure*}

Depending on how we define the matrix embedding $\mathbold{M}$,
deleting/inserting a node/edge may or may not
result in changes in the vector embeddings
of the other nodes.
As an example, consider Figure~\ref{fig:1}
that shows
a directed graph $G$, a node deletion operation (Figure~\ref{fig:11}),
and a node insertion operation (Figure~\ref{fig:12}).
In this example, the matrix embedding of the graph is
defined as follows:
for each node $i$ in $G$ we keep a row of size $n$,
where the $j^{th}$ entry is the inverse of the distance
(shortest path size)
from $i$ to node $j$.
If there is no directed path from $i$ to $j$ or if $i=j$,
the entry will be $0$.
In the figure, the row corresponding to each node shows its vector embedding,
and the matrix consisting of these rows is the matrix embedding.
For example, in the graph on the top figure (graph $G$), the vector embedding of node $1$ consists of the following elements:
$0$, $1$ and $1/2$.
First, consider deleting node $3$ from $G$.
Since for each node, we have a row and a column in the matrix embedding,
we need to delete the row and column corresponding to node $3$.
However, this is not enough as it does not always yield
a valid matrix embedding for the updated graph.
For example,
we need also to change the first element of the second row from $1/2$ to $0$.
The reason is that after deleting node $3$ and its incident edges,
there will be no path from node $2$ to node $1$.
The changes in the updated matrix embedding
are depicted with red in Figure~\ref{fig:11}.

Then as shown in Figure~\ref{fig:12}, consider inserting a new node $4$ to $G$.
In order to update the matrix embedding with respect to this change,
we need to add a column and a row corresponding to this new node.
The entries of this row and this column are depicted with red in Figure~\ref{fig:12}.
Fortunately, node $4$ and its edge to node $3$ do not change the distances between
the existing nodes of the graph.
So, rows/columns 1-3 of the matrix embedding do not need to change.
It is worth highlighting that after detecting these changes in $\mathbf M$,
matrix $\mathbf M^{\dag}$ must be updated accordingly and the updated $\mathbf M^{\dag}$ must be multiplied with $\mathbf b$
to form the updated solution.

%

\begin{definition}
\label{def:update_efficient}
Let $\mathbold{M}$ be a $n \times m$ matrix embedding of graph $G$
and $\mathsf h$ be a (complexity) function of $n$ and $m$.
We say $\mathbold{M}$ is $f$-{\em update-efficient},
iff the following conditions are satisfied:\footnote{Note that the graph and the information used to find the solution of dynamic graph regression, i.e., the pseudoinverse of matrix $\mathbf M$ and vector $\mathbf b$,
 change over time.
However, we only care about their values before and after an update operation,
as we want to find their values {\em after} the update operation,
based on their values {\em before} the update operation.
In order to keep notations as simple as possible, we do not parameterize them by time, rather,
we simply use the terms {\em before} and {\em after} the update operation to distinguish these two situations.
}
\begin{enumerate}
\item
an edge insertion/deletion or an edge weight change in $G$
result in a rank-$K$ update in $\mathbold{M}$,
where $K$ is a {\em constant}.
More precisely, if $\mathbold{M}$ and $\mathbold{M'}$ are
correct matrix embeddings before and after
one of the above-mentioned update operations in the graph,
there exist at most $K$ pairs of vectors $\mathbold{c^k}$ and $\mathbold{d^k}$
(of proper sizes) such that:
$$\mathbold{M'}=\mathbold{M}+\sum_{k=1}^K \left( \mathbold{c^k} \cdot \mathbold{d^k}^* \right).$$
We refer to each pair $\mathbold{c^k}$ and $\mathbold{d^k}$ as a pair of {\em update vectors},
and to $\sum_{k=1}^K \left( \mathbold{c^k} \cdot \mathbold{d^k}^* \right)$ as the {\em update matrix}.
\item
a node insertion in $G$ results in adding one column and/or one row to $\mathbold{M}$
and also at most a rank-$K$ update in $\mathbold{M}$.
\item
deleting the last node (i.e., the node with the largest id)
from $G$ results in deleting one column and/or one row from $\mathbold{M}$
and also at most a rank-$K$ update in $\mathbold{M}$.
\item
after an update operation in $G$, it is feasible to compute
all pairs of update vectors in $O(\mathsf h(n,m))$ time.
\item
if in $G$ any node $i$ is permuted with the last node
(i.e., with the node that has the largest id),
this can be expressed, in $O(\mathsf h(n,m))$ time,
in terms of a rank-$K$ update in $\mathbold{M}$.
\end{enumerate}
Sometimes when $\mathsf h$ is clear from the context,
we simply drop it and use the term {\em update-efficient}.
\end{definition}

\begin{remark}
\label{remark:closed}
The {\em update-efficient} property of matrix embeddings is not closed under matrix addition and matrix subtraction,
as they may increase the rank of the matrix.
\end{remark}

In Section~\ref{sec:graph_regression},
we use {\em update-efficient} matrix embeddings to develop
an efficient algorithm for dynamic graph regression.
Some well-known matrix embeddings belong to
this class of matrix embeddings.
In Lemma~\ref{lemma:update-efficient}
we show that (weighted) adjacency matrix is a $n$-update efficient matrix embedding.
Moreover, in Lemma~\ref{lemma:laplacian} we show that under some conditions,
the other  widely used matrix embedding,
i.e., the Laplacian matrix, provides a $n$-update efficient matrix embedding, too.

\begin{lemma}
\label{lemma:update-efficient}
Assume that $G$ is a simple, weighted and directed graph
and its matrix embedding $\mathbold{M}$ is defined as its weighted adjacency matrix.
$\mathbold{M}$ is a $n$-update-efficient matrix embedding
(i.e., $\mathsf h(n)=n$).
\end{lemma}
\begin{proof}
We show that $\mathbold{M}$ satisfies all the five conditions
stated in Definition~\ref{def:update_efficient}.
\begin{enumerate}
\item
When an edge is inserted/deleted between nodes $i$ and $j$
or its weight changes,
only the $ij^{th}$ entry of $\mathbold{M}$ is revised.
Let $q$ denote the amount of this change in $\mathbold{M}[ij]$,
which can be either positive or negative.
To express it in terms of a pair of update vectors $\mathbold{c}$
and $\mathbold{d}$,
we only need to define e.g., $\mathbold{c}$ as a vector
(of size $n$)
whose all elements are $0$
except the $i^{th}$ element, which is $q$;
and $\mathbold{d}$ as a vector (of size $n$) whose all elements are $0$,
except the $j^{th}$ element which is $1$.
Obviously, $\mathbold{c} \cdot \mathbold{d}^*$ yields a $n \times n$ matrix whose all elements,
except the $ij^{th}$ element, are $0$;
and its $ij^{th}$ element is $q$.
Therefore, condition~(1) of Definition~\ref{def:update_efficient}
is satisfied.
\item
When a new node $i$ is added to $G$,
we add a new column for it in $\mathbold{M}$,
which contains the weights of its incoming edges
(i.e., if there is an edge from a node $j$ to $i$,
we put the weight of this edge in the $j^{th}$ entry of the new column).
Also, we add a new row for it in $\mathbold{M}$,
which contains the weights of its outgoing edges.
Furthermore, in $\mathbold{M}$ we only need to update
those entries
whose column number or row number are $i$.
These entries are updated during column/row addition.
Hence, no update vector is required.
As a result, condition~(2) of Definition~\ref{def:update_efficient}
is satisfied.
\item
When the last node is deleted from $G$,
we delete its corresponding column and row from $\mathbold{M}$.
Furthermore, in $\mathbold{M}$ we only need to update those entries
whose column number or row number are $i$.
These entries are already deleted during column/row deletion.
Therefore, no update vector is required
and hence, condition~(3) of Definition~\ref{def:update_efficient}
is satisfied.
\item
In the case of edge insertion/deletion and edge weight change,
the update vectors $\mathbold{c}$ and $\mathbold{d}$ can be computed in $O(n)$ time.
In the case of the other update operations,
there is no need for update vectors.
As a result, condition~(4) of Definition~\ref{def:update_efficient}
is satisfied.
\item
When in $G$ we permute a node $i$ with the last node,
in $\mathbold{M}$ we only need to first exchange the $i^{th}$ column with the last column
and then in the resulted matrix,
exchange the $i^{th}$ row with the last row.
These two changes in $\mathbold{M}$
can be expressed in terms of a rank-$K$ update matrix, as follows.
We focus on exchanging the $i^{th}$ column with the last column,
as exchanging the $i^{th}$ row with the last row can be done in a similar way.
Let $l$ denote the index of the last column.
First, note that we may consider exchanging
the $i^{th}$ column with the $l^{th}$ column
as adding to each entry $ji$ in the $i^{th}$ column
the value $-\mathbold{M}[ji]+\mathbold{M}[jl]$;
and adding to each entry $jl$ in the $l^{th}$ column
the value $-\mathbold{M}[jl]+\mathbold{M}[ji]$.
Now let focus on the $i^{th}$ column
(a similar procedure can be used for the $l^{th}$ column).
We want to express the additions to the $i^{th}$ column in terms of
a pair of update vectors $\mathbold{c}$ and $\mathbold{d}$.
We can define $\mathbold{c}$ as a vector whose $j^{th}$ entry
contains $-\mathbold{M}[ji]+\mathbold{M}[jl]$;
and $\mathbold{d}$ as a vector whose all entries,
except the $i^{th}$ entry, are $0$ and its $i^{th}$ entry is $1$.
Clearly, $\mathbold{c} \cdot \mathbold{d}^*$ yields a matrix
whose $i^{th}$ column includes the values $-\mathbold{M}[ji]+\mathbold{M}[jl]$,
for $1 \leq j \leq n$,
and its other entries are $0$.
As a result, exchanging the $i^{th}$ column with the $l^{th}$ column
can be done using $2$ rank-$1$ pairs of update vectors.
Moreover, computing vectors $\mathbold{c}$ and $\mathbold{d}^*$
can be done in $O(n)$ time.
Hence, condition~(5) of Definition~\ref{def:update_efficient}
is satisfied.
\end{enumerate}
\qed
\end{proof}

\begin{lemma}
\label{lemma:laplacian}
Suppose that $G$ is a weighted, undirected and bounded-degree graph
and its matrix embedding $\mathbold{M}$ is defined as its weighted Laplacian matrix.
$\mathbold{M}$ is a $n$-update-efficient matrix embedding\footnote{Note that when inserting a new node to a
bounded-degree graph,
at most a constant (bounded) number of edges are drawn between the new node and existing nodes.}.
\end{lemma}
\begin{proof}
Assume that the degrees of the nodes of $G$ are bounded by
a constant $C$.
\begin{enumerate}
\item
When an edge is inserted/deleted between nodes $i$ and
$j$ or its weight changes,
the entries $ij$, $ji$, $ii$ and $jj$ of $\mathbold{M}$ might change.
For each of these entries,
similar to the first case of the proof of Lemma~\ref{lemma:update-efficient},
we can express the change
in terms of a pair of update vectors $\mathbold{c}$ and $\mathbold{d}$.
Hence, all these changes can be stated
in terms of an update matrix whose rank is at most $4$
(which is generated by the sum of four rank-$1$ matrices) and as a result,
condition~(1) of Definition~\ref{def:update_efficient}
is satisfied.
\item
When a new node $i$ is added to $G$, we add a new column
for it in $\mathbold{M}$,
whose $j^{th}$ entry is $-q$,
where $q$ is the weight of the edge between
nodes $i$ and $j$ (if there is no edge between
$i$ and $j$, $q$ is $0$).
In a similar way, we add a new row for it in $\mathbold{M}$.
Furthermore, in $\mathbold{M}$ we
need to increase by $1$ each entry $jj$ such that $j$ has an edge to $i$.
Since the degrees are bounded by $C$,
the number of such revisions is at most $C$.
Each such revision, can be expressed by
a pair of update vectors $\mathbold{c}$ and $\mathbold{d}$,
where the $j^{th}$ entries of these vectors are $1$
and the other entries are $0$.
These $C$ (rank-$1$) update vectors yield
an update matrix whose rank is at most $C$, as a result,
condition (2) of Definition~\ref{def:update_efficient} is satisfied.
\item
When the last node of $G$ is deleted, we delete its corresponding
column and row from $\mathbold{M}$.
Furthermore, in $\mathbold{M}$ we need to decrease
by $1$ each entry $jj$ such that $j$ has an edge to $i$.
Similar to the case of node addition,
this can be done by at most $C$ (rank-$1$) update vectors,
where the non-zero element of each
vector $\mathbold{c}$ is $-1$, rather than $1$.
As a result, condition~(3) of Definition~\ref{def:update_efficient}
is satisfied.
\item
For all the update operations,
the update vectors $\mathbold{c}$ and $\mathbold{d}$ can be computed in
$O(Cn)=O(n)$ time.
As a result, condition~(4) of Definition~\ref{def:update_efficient}
is satisfied.
\item
When in $G$ we permute a node $i$ with the last node,
in $\mathbold{M}$ we need to update both $\mathbold{D}$ and $\mathbold{W}$.
On the one hand, $\mathbold{W}$ can be updated in a way similar to
the permutation case of the proof of Lemma~\ref{lemma:update-efficient}
and this can be done in $O(n)$ time.
On the other hand, by permutation,
weighted degrees
of the nodes, except node $i$ and the last node, do not change.
Therefore $\mathbold{D}$ can be easily updated by exchanging the entry $ii$ and
the entry in the last row and last column of $\mathbold{D}$,
which can be done by two pairs of update vectors.
This can be done in $O(n)$ time.
Hence, condition~(5) of Definition~\ref{def:update_efficient}
is satisfied.
\end{enumerate}
\qed
\end{proof}

Note that weighted Laplacian matrix
(as well as unweighted Laplacian matrix),
without the mentioned constraint on degrees,
are not update-efficient, for any arbitrary complexity function $\mathsf h$.
The reason is that without the mentioned constraint,
node addition/deletion may change the degrees of $\Theta(n)$ nodes,
which then may require a rank-$\Theta(n)$ update matrix
(or $\Theta(n)$ pairs of update vectors).
Nevertheless,
if we restrict update operations,
we might be able to use Laplacian matrix
for the regression of {\em general} dynamic graphs.
For example, if we limit update operations
to {\em edge insertion}, {\em edge deletion}
and {\em edge weight change},
it will not be necessary for nodes of the graph
to have a bounded degree.

\section{Dynamic graph regression under least squares error}
\label{sec:graph_regression}

In this section, we condition on the existence of {\em update-efficient}
matrix embeddings and show that
the exact optimal solution of $l_2$ graph regression
can be updated in a time much faster than computing it from scratch.

At the high-level, the algorithm of
solving $l_2$ dynamic graph regression
consists of three phases:
1) the {\em matrix embedding} phase, wherein
we compute an update-efficient matrix embedding $\mathbold{M}$
for the given graph (we assume that it is static),
2) the {\em pre-processing} phase,
wherein we assume that we are given a static graph
and we find a solution for it, and
3) the {\em update} phase, wherein after any update operation in $G$,
$\mathbold{M}$ and
the already found solution are revised to become valid for the new graph.
For the pre-processing phase, we
use $\mathbold{M}^{\dag} \cdot \mathbold{b}$ as the optimal solution.
Naively computing the SVD and pseudoinverse of $\mathbold{M}$
requires $O\left(\min(nm^2,mn^2)\right)$ time.
However using fast matrix multiplication~\cite{DBLP:conf/issac/Gall14a} (which is not practical!),
it can be done in $O(nm^{1.3728})$ time.
In the proof of Theorem~\ref{theorem:update-efficient},
we discuss how the optimal solution is updated,
after an update operation in $G$.

\begin{theorem}
\label{theorem:update-efficient}
Let $\mathbold{M}$ be a $n\times m$ {\em update efficient} matrix embedding
of graph $G$.
After a pre-processing phase which takes
$O \left(\min(nm^2,n^2m) \right)$ time,
after any update operation in the graph,
the exact optimal solution of the $l_2$ graph regression problem for the updated graph
can be computed in $O(nm)$ time.
\end{theorem}
\begin{proof}
After an update operation in $G$,
we require to first update $\mathbold{M}^{\dag}$ and then,
update $\mathbold{M}^{\dag} \cdot \mathbold{b}$.
The way of updating $\mathbold{M}^{\dag}$ depends on the update operation done in the graph.

\begin{itemize}
\item{Edge insertion/deletion or edge weight change:}
in any of these cases, due to the {\em update-efficient} property of
$\mathbold{M}$,
we have a sequence of at most $K$ {\em rank-$1$} updates:
$$\mathbold{M^{k+1}}=\mathbold{M^k}+\mathbold{c^k} \cdot {\mathbold{d^k}}^*,$$
for $1 \leq k < K$,
where
$\mathbold{c^k}$ and ${\mathbold{d^k}}$ are a pair of update vectors,
$\mathbold{M^1}=\mathbold{M}$
and $\mathbold{M^K}$ is the correct matrix embedding of $G$
after the update operation.
After each {\em rank-$1$} update $\mathbold{M^{k+1}}=\mathbold{M^k}+\mathbold{c^k} \cdot {\mathbold{d^k}}^*$,
we may exploit e.g., the algorithm of Meyer~\cite{Generalized_Inverse_2}
that given a matrix $\mathbold{A}$ and its Moore-Penrose pseudoinverse
$\mathbold{A}^{\dag}$
and a pair of update vectors $\mathbold{c}$ and $\mathbold{d}$,
computes Moore-Penrose pseudoinverse of $(\mathbold{A}+\mathbold{c} \cdot \mathbold{d}^*)$.
Due to many notations and a long explanation required to introduce
this method, we here omit its description and
refer the interested reader to \cite{Generalized_Inverse_2}.
The key point is that given $\mathbold{A}^{\dag}$,
computing $(\mathbold{A}+\mathbold{c} \cdot \mathbold{d}^*)^{\dag}$
can be done in $O(nm)$ time.
Therefore and after applying this algorithm for at most $K$ times,
we can compute the Moore-Penrose pseudoinverse of the matrix embedding
of the updated graph in
$O(Knm)=O(nm)$ time.

\item{Node insertion:}
in this case, we need to follow a two-phase procedure.
In the first step, we require to append a row and (if needed) a column to $\mathbold{M}$
that correspond to the new node and carry out its embedding information.
Let's focus on appending a new column (appending a new row can be dealt with in a similar
way).
Speaking precisely,
we have matrix
$$\mathbold{M'}= [\mathbold{M} \;\;\;\;\; \mathbold{a}],$$
where $\mathbold{a}$ is the column corresponding to the new node
and $\mathbold{M}$ is the matrix embedding of $G$ before the update operation.
We want to compute $\mathbold{M'}^{\dag}$ based on $\mathbold{M}^{\dag}$ and
$\mathbold{a}$.
For this, we may use e.g., the Greville's algorithm \cite{Greville},
which is as follows.
Let $\mathbold{d} = \mathbold{M}^{\dag} \cdot \mathbold{a}$,
$\mathbold{c} = \mathbold{a} - \mathbold{M} \cdot \mathbold{d}$,
and $\mathbold{\theta}$ be the null matrix (of proper size), i.e.,
the matrix whose all elements are zero.
Then
\begin{equation*}
\mathbold{M'}{^\dag} = \left[
\begin{matrix}
\mathbold{M}^{\dag} - \mathbold{d} \cdot \mathbold{f} \\
\mathbold{f} \\
\end{matrix}
\right]
\end{equation*}
where
\begin{equation}
\mathbold{f} =
\begin{cases}
\mathbold{c}^{\dag}, & \text{if } \mathbold{c} \neq \mathbold{\theta} \\
(1+\mathbold{d}^* \cdot \mathbold{d})^{-1} \mathbold{d}^* \cdot \mathbold{M}^{\dag} , & \text{if } \mathbold{c} = \mathbold{\theta}
\end{cases}
\end{equation}

Since adding the new node may affect the vector embeddings of the existing nodes
by a sequence of at most $K$ rank-$1$ update vectors, in the second step,
we need to reflect these changes to the matrix embedding of $G$.
In other words, for at most $K$ pairs
of vectors $\mathbold{c^k}$ and $\mathbold{d^k}$,
we have:
$$\mathbold{M^{k+1}}=\mathbold{M^k}+\mathbold{c^k} \cdot {\mathbold{d^k}}^*,$$
where $\mathbold{M^1}$ is $\mathbold{M'}$
and $\mathbold{M^K}$ is the correct matrix embedding of $G$
after the node insertion.
Hence, similar to the previous case,
we may use the algorithm of Meyer~\cite{Generalized_Inverse_2}
to compute
${\mathbold{M^{k+1}}}^{\dag}$ based on ${\mathbold{M^k}}^{\dag}$.
Each of these two steps takes $O(nm)$ time.

\item{Node deletion:}
in this case,
we need to follow a three-phase procedure.
In the first step,
we perform a permutation on $\mathbold{M}$
so that the node that we want to delete
becomes the last node in the matrix
(i.e., it becomes the node with the largest id).
Note that updating $\mathbold{M}$ with this new permutation of nodes
may need to call a sequence of at most $K$ rank-$1$ update vectors.
Hence,
we may need to use Meyer's algorithm \cite{Generalized_Inverse_2} to compute
each ${\mathbold{M^{k+1}}}^{\dag}$ based on ${\mathbold{M^{k}}}^{\dag}$ and the update vectors, where ${\mathbold{M^{1}}}=\mathbold{M}$.
In the second step,
we need to delete a row and (if needed) a column from $\mathbold{M}$
that correspond to the deleted node and carry out its embedding information.
Let's again focus on deleting a column (as deleting a row can be done in a similar way).
We have matrix $\mathbold{M'}$ such that
$$\mathbold{M}= [\mathbold{M'} \;\;\;\;\; \mathbold{a}],$$
where $\mathbold{a}$ is the column corresponding to the deleted node
and $\mathbold{M}$ is matrix embedding of $G$ before the update operation.
We may again use the Greville's algorithm \cite{Greville}
to compute $\mathbold{M'}^{\dag}$ based on $\mathbold{M}^{\dag}$ and $\mathbold{a}$.
Finally,
since deleting a node may change the vector embeddings of the existing nodes by a
sequence of at most $K$ rank-$1$ update vectors,
in the third step,
we need to apply these changes to the matrix embedding of the graph.
Hence and similar to the previous cases,
we can use Meyer's algorithm \cite{Generalized_Inverse_2} to compute
each ${\mathbold{M^{k+1}}}^{\dag}$ based on ${\mathbold{M^{k}}}^{\dag}$ and the update vectors, with ${\mathbold{M^{1}}}=\mathbold{M'}$.
Each of these three steps takes at most $O(nm)$ time.
\end{itemize}
As a result, after an update operation in $G$,
$\mathbold{M}^{\dag}$ can be updated in $O(nm)$ time.
Note that in the case of node deletion,
when we perform a permutation on the nodes of $G$ and rows of $\mathbold{M}$,
we need also consistently permute the elements of $\mathbold{b}$ and then,
remove from $\mathbold{b}$ the measured value of the deleted node
(which after the permutation will be the last element of $\mathbold{b}$).
These operations can be done in $O(n)$ time.
Furthermore, after a node insertion,
we need to add the measured value of the new node to the end of $\mathbold{b}$,
which can be done in $O(1)$ time.
A naive multiplication of the updated $\mathbold{M}^{\dag}$
and the updated $\mathbold{b}$ yields
the optimal solution of the updated graph
and it takes only $O(nm)$ time.
\qed
\end{proof}


\begin{corollary}
\label{corollary1}
Given a graph $G$,
if its matrix embedding is defined as
its (weighted or unweighted) adjacency matrix,
after any update operation in $G$
(i.e., node insertion or node deletion or edge insertion or edge deletion or edge weight change),
the exact optimal solution of
the $l_2$ dynamic graph regression problem
can be updated
in $O(n^2)$ time.
\end{corollary}
\begin{proof}
 Lemma~\ref{lemma:update-efficient} says that the (weighted) adjacency matrix of $G$
 is a $n$-update-efficient and therefore, a $n^2$-update-efficient matrix embedding of $G$.
 Hence and as Theorem~\ref{theorem:update-efficient} says,
 after any update operation in the graph,
 the optimal solution of the $l_2$ graph regression problem for the
 updated graph can be computed in $O(n^2)$ time.
 \qed
\end{proof}

To the best of our knowledge,
Corollary~\ref{corollary1} provides the first result on updating the exact optimal solution of linear regression
over general dynamic graphs,
in a time less than computing it from scratch.
Note that if the weighted adjacency matrix of $G$ is used as the matrix embedding,
computing the optimal solution
of dynamic graph regression from scratch
will take $O(n^3)$ time.

As mentioned earlier, if we restrict the update operations to
{\em edge insertion}, {\em edge deletion} and
{\em edge weight change},
the Laplacian matrix will
be a $n^2$-update-efficient matrix embedding.
Hence, we will have the following result.

\begin{corollary}
\label{corollary2}
Suppose that we are given an undirected graph $G$,
whose matrix embedding is defined as
its (weighted or unweighted) Laplacian matrix.
\begin{enumerate}
\item
After an edge insertion or an edge deletion
or an edge weight change in $G$,
we can update the exact optimal solution of
the $l_2$ dynamic graph regression problem
in $O(n^2)$ time.
\item
If $G$ is a bounded-degree graph,
after any update operation in $G$,
i.e., node insertion (wherein a bounded number of edges are drawn between the new node and the existing nodes)
or node deletion or edge insertion or edge deletion or edge weight change,
the exact optimal solution can be updated
in $O(n^2)$ time.
\end{enumerate}
\end{corollary}

In the end, it might be of interest to see whether similar results can be obtained for other regression models.
Someone may start with linear regression under {\em least absolute deviation}, as it is the closest setting to our studied problem,
and then go beyond linear models and study nonlinear regression functions for dynamic graphs.
In dynamic graph regression under {\em least absolute deviation}, the distance is measured using the $l_1$ norm.
More precisely, we want to solve the following problem:
\begin{equation}
\label{eq:deviation}
argmin_{\mathbold{x}} ||\mathbold{A} \cdot \mathbold{x} - \mathbold{b}||_1,
\end{equation}
where the $l_1$ norm of a vector $\mathbf v$ is defined as follows:
$||\mathbf v||_1 = \sum_i |\mathbf v_i|$.
What makes this problem and most of the other nonlinear regression problems theoretically more challenging,
is that they do not have a closed form solution~\cite{10.1561/0400000060}, i.e., a solution like what is presented in Equation~\ref{eq:closedform} for our studied problem.
However, we believe it might be possible to derive similar results for some of these models,
and leave it as an interesting direction for future work.

\section{Experimental results}
\label{sec:experiments}

In this section, we empirically evaluate the update algorithm.
First in Section~\ref{sec:empirical}, we report the results of our experiments over synthetic and real-world graphs.
Then in Section~\ref{sec:casestudy}, we discuss the applicability and usefulness of
dynamic graph regression and its update algorithm,
in a real-world case study.

\subsection{Empirical evaluation}
\label{sec:empirical}

In this section,
we examine the empirical efficiency of updating the solution of dynamic graph regression,
compared to computing it from scratch.
Since node insertion and node deletion include edge insertion and edge deletion,
we only consider these two update operations.
In node insertion, a new node along with some edges connecting it to the existing nodes are added to the graph.
In node deletion, the node with the largest id is deleted from the graph.
Each of these operations is conducted for $50$ times and the average results are reported.
We define the matrix embedding of the graph as its adjacency matrix.
We run  both the update algorithm and the algorithm
of computing the solution from scratch, and compare their times.

\begin{figure*}
\centering
\subfigure[Node insertion.]
{
\includegraphics[scale=1.3]{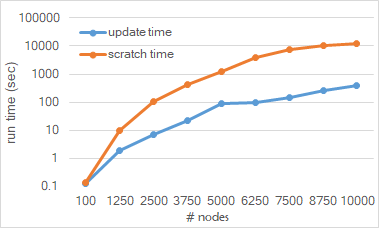}
\label{fig:add}
}
\quad
\subfigure[Node deletion.]
{
\includegraphics[scale=1.3]{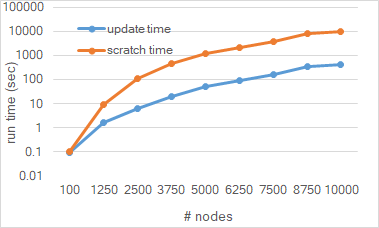}
\label{fig:remove}
}
\caption
{
\label{fig:evaluation}
Empirical evaluation over synthetic graphs.
The vertical axes are presented in the logarithmic scale.
}
\end{figure*}

\subsubsection{Synthetic graphs}
We study how
the update algorithm scales with respect to the size of the graph,
compared to the algorithm that computes the solution from scratch.
We generate several ER random graphs~\cite{jrnl:Erdos} with $100$, $1250$, $2500$,
$3750$, $5000$, $6250$, $7500$, $8750$ and $10000$ nodes,
all with the same edge induction probability $0.5$.
Figures~\ref{fig:add} and \ref{fig:remove} present
the results over the generated graphs, respectively for node insertion and node deletion.
Due to large differences between update and scratch times,
in the charts of these figures the vertical axes are presented in the logarithmic scale.
As depicted in the figures,
incrementally updating the solution
rather than computing it from scratch takes much less time.
Moreover, it is more scalable: it can handle large graphs within still an empirically tractable time.
The time of computing the solution from
scratch quickly increases by increasing the size of the graph,
so that performing it for large graphs is
intractable in practice.

\subsubsection{Real-world graphs}

\begin{table*}
\caption{Empirical evaluation over real-world graphs. \label{table:r1}
}
\centering
\begin{tabular}{ l| l l| l l }
\hline
\multirow{2}{*}{Dataset}& \multicolumn{2}{|c|}{node insertion} & \multicolumn{2}{|c}{node deletion} \\
\cline{2-5}
&  update time  & scratch time   & update time  & scratch time \\
\hline
{\em wiki-vote}             & {\bf 193.76} & 1601.20  &  {\bf 159.37} & 1451.28  \\
{\em lastfm-asia}           & {\bf 38.08} &  261.41 &  {\bf 31.06} & 256.83  \\
{\em soc-sign-bitcoinotc}   & {\bf 23.94} & 102.99 &  {\bf 21.31}  &  101.14    \\
\hline
\end{tabular}
\end{table*}

We also evaluate the algorithms over the following real-world graphs:
{\em wiki-vote}~\cite{DBLP:conf/www/LeskovecHK10}\footnote{\url{https://snap.stanford.edu/data/wiki-Vote.html}},
{\em lastfm-asia}~\cite{feather}\footnote{\url{https://snap.stanford.edu/data/feather-lastfm-social.html}}
and {\em soc-sign-bitcoinotc}~\cite{kumar2016edge}\footnote{\url{https://snap.stanford.edu/data/soc-sign-bitcoin-otc.html}}.
{\em Wiki-vote} has 7115 nodes and 103689 edges,
{\em lastfm-asia} has 7624 nodes and 27806 edges,
and {\em soc-sign-bitcoinotc} has 5881 nodes and 35592 edges.
The comparison results are presented in Table~\ref{table:r1}.
They depict that updating the solution of linear regression
is considerably more efficient than computing it from scratch.

We note than while {\em wiki-vote} and {\em lastfm-asia}
have almost the same number of nodes ({\em lastfm-asia} has slightly more nodes),
over {\em wiki-vote} both update time and scratch time are much larger.
The reason is that {\em wiki-vote} is a considerably more dense graph than {\em lastfm-asia}.
When computing the solution from scratch,
we need to compute the pseudoinverse of the adjacency matrix of the updated graph.
This operation is much more time consuming for dense graphs than sparse graphs.
On the other hand, for dense graphs more pairs of update vectors are induced, after the update operation.
Moreover, updating the pseudoinverse of the matrix is more time consuming.
As a result, update time over a graph such as {\em wiki-vote} is larger
than update time over a sparser graph such as {\em lastfm-asia}.



\subsection{A case study}
\label{sec:casestudy}


Cryptocurrency price prediction is a challenging task in market data analysis.
Correlations between cryptocurrencies~\cite{report:graph_regression}
can be useful in predicting cryptocurrency prices.
In this section, we investigate applying a linear regression model
to predict cryptocurrency prices.
From the \texttt{coinmetrics} website{\footnote{\url{https://charts.coinmetrics.io}}},
we collect the correlation information and the prices (in the US Dollar) of
the following cryptocurrencies, in July $28^{th}$ 2021:
Bitcoin (\textsf{BTC}), Decred (\textsf{DCR}), DigiByte (\textsf{DGB}),
Dogecoin (\textsf{DOGE}), Litecoin (\textsf{LTC}), MCO Token (\textsf{MCO}) and Vertcoin (\textsf{VTC}).
Figure~\ref{fig:case1} shows the weighted graph of  the correlations,
where the nodes are the currencies and the weight of an edge represents the correlation between its two endpoints.
Figure~\ref{fig:case2} presents the prices of the cryptocurrencies.
Our goal is to develop a linear  model that predicts the price of a cryptocurrency,
in the form of a linear combination of its correlations with the other cryptocurrencies.
Since the correlations and the prices are dynamic and time-variant,
we want to update the model by the changes in data.

\begin{figure*}
\centering
\subfigure[The weighted graph of cryptocurrency correlations.]
{
\includegraphics[scale=0.4]{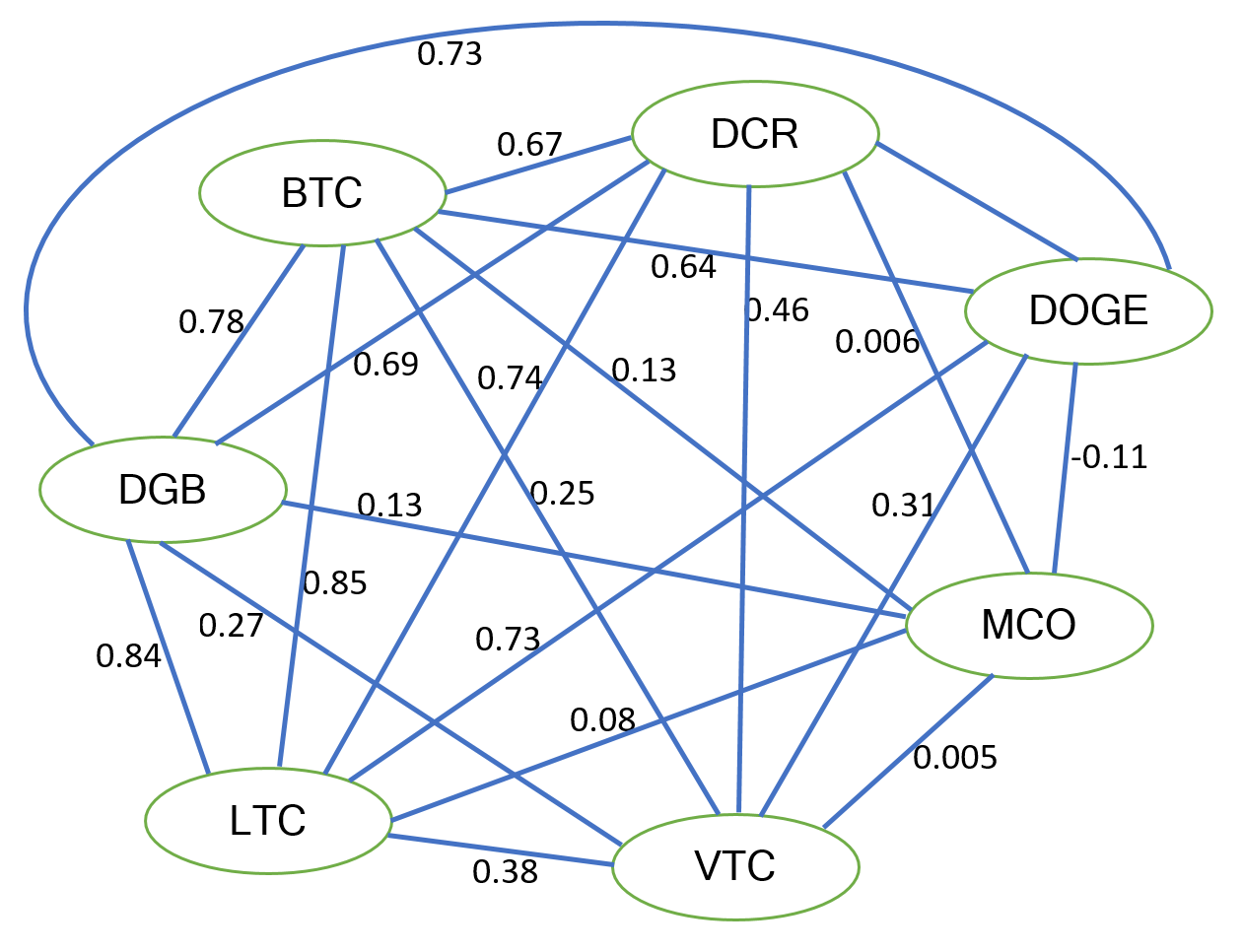}
\label{fig:case1}
}
\quad\quad\quad\quad\quad
\subfigure[Cryptocurrency prices.]
{
\includegraphics[scale=0.4]{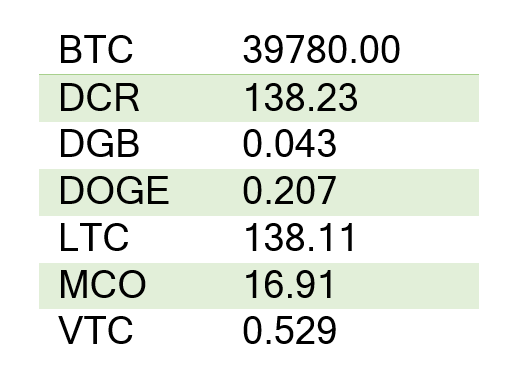}
\label{fig:case2}
}
\caption
{
\label{fig:casestudy}
The graph of cryptocurrency correlations and the prices.
}
\end{figure*}

To do so, we  consider the weighted adjacency matrix of the correlation graph of Figure~\ref{fig:case1}
as the matrix embedding $\mathbold M$, and the values presented in Figure~\ref{fig:case2} as the vector $\mathbold b$.
We solve the corresponding linear regression problem (under $l_2$ norm).
Its exact optimal solution is:
\begin{equation*}
\centering
\mathbold x = \mathbold M^{\dag} \mathbold b =
\begin{bmatrix} -36497 \\ 12368 \\ 14436 \\ -15438 \\ 9706.3 \\ 103000 \\32985 \end{bmatrix}.
\end{equation*}
This means that in the linear regression, the weight of the correlation with \textsf{BTC} is $-36497$,
the weight of the correlation with \textsf{DCR} is $12368$,
the weight of the correlation with \textsf{DGB} is $14436$
and so on.

When the correlation graph or vector $\mathbold b$ change over time,
we need to accordingly update the weights.
We note that as stated in Section~\ref{sec:update-efficient},
weighted adjacency matrix used in this case study is an update efficient matrix embedding.
So, the weights of the linear regression can be updated efficiently.
An example change
is the decrease of the correlation between \textsf{BTC} and \textsf{DGB} from $0.7891$ to $ 0.7682$.
After updating our linear regression model with respect to this change,
the weights in vector $\mathbold x$ become:
\begin{equation*}
\centering
\mathbold x =
\begin{bmatrix} -37069 \\ 12196 \\ 15248 \\ -15252 \\ 9553.5 \\ 102180 \\ 32632 \end{bmatrix}.
\end{equation*}

\section{Conclusion}
\label{sec:conclusion}

In this paper, we
studied the theory of linear regression over
dynamic graphs.
First, we presented the class of {\em update-efficient} matrix embeddings,
that defines conditions sufficient for a matrix embedding
to be used for least squares dynamic graph regression.
Then, we showed that given a $n \times m$
update-efficient matrix embedding (e.g., the adjacency matrix),
after an update operation in the graph,
the exact optimal solution of graph regression
can be updated in $O(nm)$ time.
Finally by conducting experiments over synthetic and real-world graphs,
we showed the high empirical efficiency of the update algorithm.


\bibliographystyle{spmpsci}
\bibliography{allpapers}

\end{document}